\documentclass{amsart}
\usepackage{microtype,amssymb,algorithm,algpseudocode}

\def\QQ{\ensuremath{\mathbb{Q}}}
\def\RR{\ensuremath{\mathbb{R}}}
\newcommand{\CC}{{\mathbb C}}
\newcommand{\PP}{{\mathbb P}}

\newcommand{\rank}{\textup{rank}}

\newtheorem{theorem}{Theorem}
\newtheorem{definition}[theorem]{Definition}
\newtheorem{lemma}[theorem]{Lemma}
\newtheorem{remark}[theorem]{Remark}
\newtheorem{example}[theorem]{Example}
\newtheorem{corollary}[theorem]{Corollary}
\newtheorem{question}[theorem]{Question}

\newcommand{\FF}{{\mathbb{F}}}
\newcommand{\R}{{\mathbb{R}}}

\newcommand{\dis}{\displaystyle}

\newcommand{\wh}{\widehat}

\newcommand{\f}{\frac}

\newcommand{\ra}{\rightarrow}

\newcommand{\degree}{\text{degree}}

\title{Certifying the Existence of Epipolar Matrices}

\author[Agarwal]{Sameer Agarwal}
\address{Google Inc.}
\email{sameeragarwal@google.com}

\author[Lee]{Hon-Leung Lee}
\address{Department of Mathematics, University of Washington, Seattle, WA 98195}
\email{hllee@uw.edu}

\author[Sturmfels]{Bernd Sturmfels}
\address{Department of Mathematics, University of California, Berkeley, CA 94720}
\email{bernd@berkeley.edu}

\author[Thomas]{Rekha R. Thomas}
\address{Department of Mathematics,  University of Washington, Seattle, WA 98195}
\email{rrthomas@uw.edu}

\thanks{Lee and Thomas were partially supported by NSF grants DMS-1115293 and DMS-1418728. Sturmfels was partially supported by NSF grant DMS-0968882.}

\begin{document}

\begin{abstract}
Given a set of point correspondences in two images, the existence of a fundamental matrix is a necessary condition for the points to be the images of a 3-dimensional scene imaged with two pinhole cameras. If the camera calibration is known then one
requires the existence of an essential matrix.

We present an efficient algorithm, using exact linear algebra,
for testing the existence of a fundamental matrix.
The input is  any number of point correspondences.
For essential matrices, we characterize  the solvability of the
Demazure polynomials.
In both scenarios, we determine
which linear subspaces  intersect
 a fixed set defined by non-linear polynomials.
The  conditions we derive are polynomials stated purely in terms of image coordinates.
They represent a new class of two-view invariants, free of fundamental (resp.~essential)~matrices.
\end{abstract}

\maketitle

\section{Introduction}
\label{sec:introduction}

Computer vision motivates
the following basic question in multiview geometry:

\begin{question}
\label{question1}
 Given a set of point correspondences $\{ (x_i, y_i) \in \RR^2 \times \RR^2, \,\,i=1,\ldots, m \}$, are these the (corresponding) images  in two pinhole cameras of $m$ points in $\R^3$?
\end{question}

Depending on whether we know the intrinsic calibration of the cameras or not, the classic answer to this question is a necessary condition in terms of either the {\em essential matrix}~\cite{longuethiggins81} or the {\em fundamental matrix}~\cite{faugeras-92,faugeras-luong-maybank-92,hartley-uncalibrated-relative-92, hartley-uncalibrated-stereo-92} and it goes as follows:

A set of point correspondences $\{ (x_i, y_i) \in \RR^2 \times \RR^2, \,\,i=1,\ldots, m \}$
comes from $m$ points in $\R^3$ with two uncalibrated (resp.~calibrated) cameras
   only if there exists a fundamental matrix $F$
   (resp.~essential matrix $E$)  such that the $(x_i,y_i)$
           satisfy the {\em epipolar constraints}~\cite[Chapter 9]{hartley-zisserman-2003}. A fundamental matrix $F$ is any real $3 \times 3$ matrix of rank two,
         and an essential matrix $E$ satisfies nine additional cubic {\em Demazure polynomials}. Under mild genericity conditions on the point correspondences, the existence of these matrices is also sufficient for
the correspondences $(x_i,y_i)$ to be the images of a 3D scene.

This however is not a satisfying answer, as it has replaced one existence question about a
three-dimensional scene and pair of pinhole cameras with another about the existence of a $3 \times 3$ matrix. Thus, to answer Question~\ref{question1}, we now have to answer:
\begin{question}
\label{question2}
Given a set of $m$ point correspondences $(x_i,y_i)$ in $ \RR^2 \times \RR^2$,
does there exist a fundamental (essential) matrix relating them via the epipolar constraints?
\end{question}

The answer to Question~\ref{question2} is known to varying degrees in several special cases,
but a complete answer has remained elusive.
For example, in the calibrated situation, when $m=5$, there exists up to 10 distinct and possibly all
complex essential matrices~\cite{demazure}. In the uncalibrated scenario, when $m\le7$ there always exists a real rank deficient matrix satisfying the epipolar
constraints~\cite{hartley1994projective,sturm1869problem}. This matrix
is not guaranteed to have rank two and therefore may not be a fundamental matrix. Thus, the so-called {\em Seven Point Algorithm} is not guaranteed to return a fundamental matrix as a solution.
Another result
for $m=7$, by Chum  {\em et al.}~\cite{matas-et-al-2005}, states that
 if five or more points are related by a homography, then there always exists a fundamental matrix for all of them.

In this paper we give a complete answer to Question~\ref{question2} for fundamental matrices:
we determine the exact conditions
on $(x_1,y_1),\ldots,(x_m,y_m)$ under which a real rank two matrix satisfying the epipolar constraints exists.
They yield an efficient algorithm that takes the point pairs as input and verifies whether they have a fundamental matrix. To our knowledge such a direct test on point pairs does not exist in the literature. The usual method is to use the point pairs to first find a fundamental matrix by solving the epipolar constraints for a rank deficient matrix. This involves non-linear algebra and methods to solve polynomial systems. Our check only needs the point pairs and linear algebra.

For essential matrices, our answer is conditioned on the rank of a certain $m \times 9$ data matrix $Z$. If $\rank(Z) \le 3$ or $\rank(Z) \ge 7$ we give precise conditions for the existence of a real essential matrix. If $ 4 \le \rank(Z) \le 6$~we
characterize, for all values of $m$, when the Demazure polynomials~\cite{demazure}
 have a (possibly complex) solution.

The key technical task is to establish conditions for the intersection of a linear subspace with a fixed set that is described in terms of polynomials.
That fixed set is the set of  real rank two matrices (in the uncalibrated case),
or the  variety of essential matrices (in the calibrated case).
  The existence conditions we derive are in terms of polynomials in the $2m$ coordinates of
 $ (x_1, y_1), \ldots, (x_m, y_m)$. These polynomials  represent a new class of two-view invariants, free of fundamental (essential) matrices.

In the remainder of this section we establish our notation and some basic facts about cameras and projective varieties. Section~\ref{sec:fundamental} considers the existence problem for the fundamental matrix and Section~\ref{sec:essential} does so for the essential matrix. We conclude in Section~\ref{sec:discussion} with a discussion of the results and directions for future work.

\subsection{Notation}
\label{sec:notation}
Capital roman letters (say $E, F, X, Y, Z$) denote matrices, and
upper case calligraphic letters denote sets (say $\mathcal{E},
\mathcal{F}$). For a matrix $F$, the corresponding lower case letter
$f$ denotes the vector obtained by concatenating the rows of $F$.
All polynomials have coefficients in a field $\FF$, which is usually either $\RR$ or $\CC$.
The projective space $\PP^n_\FF$ is $\FF^{n+1}$ in which we identify $u$ and $v$ if $u = \lambda v$ for some $\lambda\in \FF\setminus\{0\}$. For example $(1,2,3)$ and $(4,8,12)$ are the same
point   in $\PP^2_\RR$, denoted as $(1,2,3) \sim (4,8,12)$.
The set of $m\times n$ matrices with entries in $\FF$ is denoted by $\FF^{m\times n}$. The symbol $A^\dagger$ denotes the Moore-Penrose pseudo-inverse of $A$. For $v \in \RR^3$,
\begin{align}
[v]_\times \,\,:= \,\,
\begin{bmatrix}
0 & -v_3 & v_2\\
v_3 & 0 & -v_1\\
-v_2 & v_1 & 0
\end{bmatrix}
\end{align}
is a skew-symmetric matrix whose rank is two unless $v=0$. Also, $[v]_\times w = v \times w$, where $\times$ denotes the vector cross product.

We use $[m]$ to denote the set $\{1, \ldots, m\}$.
For $A\in \FF^{m\times n}$, we have ${\rm ker}_\FF (A) = \{u \in \FF^n: A u = 0\}$,
and ${\rm rank}(A) =  n - \dim({\rm ker}_\FF (A))$. We use $\det(A)$ to denote the determinant of $A$.
Let $q_{ij}(A)$ denote the
$2 \times 2$ minor of the $3 \times 3$ matrix $A$, obtained by dropping the $i$-th row and $j$-th column.
Points $x_i$ and $ y_i$  in $\RR^2$ will be identified with their homogenizations
$(x_{i1},x_{i2},1)^\top$ and $(y_{i1},y_{i2},1)^\top $ in $ \PP_\RR^2$.
Also, we write
$y_i^\top \otimes x_i^\top := \begin{bmatrix}
y_{i1}x_{i1} & y_{i1}x_{i2} & y_{i1} & y_{i2}x_{i1} & y_{i2}x_{i2} & y_{i2} & x_{i1} & x_{i2} & 1\end{bmatrix}\in \RR^{1\times 9}$.

\subsection{Camera Matrices}
A pinhole camera can be modeled by a real $3 \times 4$ matrix $P$ defined up to scale with $\rank(P)=3$. Partitioning a camera as $P=\begin{bmatrix} A & b \end{bmatrix}$ where $A \in
\RR^{3 \times 3}$, we say that $P$ is a {\em finite camera} if $A$ is
nonsingular, and an {\em infinite camera} otherwise.  In this paper we
restrict ourselves to finite cameras.
The unique point in  $ \PP_\RR^3$ represented by a generator $c \in \R^4$ of $\ker_\RR(P)$ is called the {\em center} of the camera $P$.
 For a finite camera, the center is
$c  = ((-A^{-1}b)^\top, 1)^\top \in\PP_\RR^3$. Two cameras $P_1$ and $P_2$ are {\em coincident} if their centers coincide; else they are {\em non-coincident}. A camera $P$
defines a linear projection
$\PP_\RR^3 \backslash \{c\} \ra \PP_\RR^2$ given by $w\mapsto Pw$. The point $Pw \in \PP_\RR^2$ is the {\em image} of $w$ in the camera $P$. By de-homogenizing
(i.e.~scaling the last coordinate to be $1$),
 we can view the image $Pw$ as a point in~$\RR^2$.

A finite camera $P$ can be written as $P = K\begin{bmatrix} R & t\end{bmatrix}$,
where $t \in \RR^3$,
 $K$ is an upper triangular matrix with positive diagonal entries known as the {\em calibration matrix}, and $R \in SO(3)$ is a rotation matrix that represents the orientation of the camera coordinate frame.
    If the calibration matrix $K$ is known, then the camera is
    said to be {\em calibrated}, and otherwise the camera is {\em uncalibrated}. The {\em normalization} of a calibrated camera $P = K\begin{bmatrix} R & t\end{bmatrix}$ is the camera $K^{-1}P = \begin{bmatrix} R & t\end{bmatrix}$. If $Pw = x$ is the image of $w$ in the calibrated camera $P$, then $K^{-1}x$ is called the {\em normalized image} of $w$, or equivalently, it is the image of $w$ in the normalized camera $K^{-1}P$. Thus if a camera $P$ is calibrated we can remove the effect of the known calibration $K$ by passing to the normalized camera $K^{-1}P$ and normalized images $K^{-1}x$.

\subsection{Epipolar Matrices}
\label{sec:epipolar}

 In this paper we use the name {\em epipolar matrix} to refer to either a {\em fundamental matrix} or {\em essential matrix}
derived from the {\em epipolar geometry} of a pair of cameras. These matrices are explained and studied in \cite[Chapter 9]{hartley-zisserman-2003}.

An {\em essential matrix} is any $3 \times 3$ real matrix of the form $E = SR$ where $S$ is a skew-symmetric matrix and $R \in SO(3)$. Essential matrices are characterized by the property that they have rank two (and hence one zero singular value) and two equal non-zero singular values. An essential matrix depends on six parameters,
three each from $S$ and $R$, but since it is only defined up to scale, it has five degrees of freedom.

The essential matrix of the two normalized cameras $\begin{bmatrix} I & 0 \end{bmatrix}$ and $\begin{bmatrix} R & t \end{bmatrix}$
is $E = [t]_\times R$.
For every pair of images $\tilde{x}$ and $\tilde{y}$ in these cameras of a point $w \in \PP_\RR^3$, the triple $(\tilde{x}, \tilde{y}, E)$ satisfies the {\em epipolar constraint}
\begin{align} \label{eq:epipolar}
\tilde{y}^\top E \tilde{x} \,\,=\,\, 0.
\end{align}
Further, any $E=SR$ is the essential matrix of a pair of cameras as shown in \cite[Section 9.6.2]{hartley-zisserman-2003}.

If the calibrations $K_1$ and $K_2$
of the two cameras were not known, then for a pair of corresponding images $(x,y)$ in the two cameras, the epipolar constraint becomes
$$  0 \,=\, \tilde{y}^\top E \tilde{x} \,=\, y^\top K_2^{-\top} E K_1^{-1} x.$$
The matrix $F := K_2^{-\top} E K_1^{-1}$
 is the {\em fundamental matrix} of the two uncalibrated cameras. This is a rank two matrix but its two non-zero singular values are no longer equal.
 Conversely, any real $3 \times 3$ matrix of rank two is the fundamental matrix of a pair of cameras \cite[Section 9.2]{hartley-zisserman-2003}. A fundamental matrix has seven degrees of freedom since it satisfies that rank two condition and is only defined up to scale.

Chapters 9 and 17 in \cite{hartley-zisserman-2003} offer other definitions of fundamental matrices. For example, consider any pair of non-coincident uncalibrated cameras $P_1$ and $P_2$ with  centers $c_1$ and $c_2$.
Then \cite[Equation 9.1]{hartley-zisserman-2003} says that
their fundamental matrix is
\begin{align}
F \,\,=\,\, [P_2c_1]_{\times} P_2 P^{\dagger}_1.
\end{align}
It is also mentioned in \cite[Section 9.6.2]{hartley-zisserman-2003} that the set of fundamental matrices can be parametrized as
 $F = [b]_\times H$, where $b$ is a non-zero vector and $H$ is an invertible matrix $3 \times 3$ matrix. We will use this in Section~\ref{sec:fundamental}.

\subsection{Projective Varieties}
\label{sec:projective}
We recall some basic notions from algebraic geometry
  \cite{cox2007ideals,eisenbud1995commutative}. Let $\FF[u] = \FF[u_1,\ldots,u_n]$ denote the ring of all polynomials with coefficients in $\FF$.
\begin{definition}[Homogeneous Polynomial]
A polynomial in $\FF[u]$  is homogeneous (or called a {\em form}) if all its monomials have the same total degree.
\end{definition}
For example, $u_1^2 u_2 + u_1 u_2^2$ is a form of degree three but $u_1^3 + u_2$ is not a form.

\begin{definition}[Projective Variety and Subvariety]
A subset $\mathcal{V} \subseteq \PP_\FF^{n}$ is a projective variety if there are homogeneous polynomials $h_1,\ldots,h_t \in \FF[u]$ such that
$\mathcal{V} = \{u \in \PP^{n}_\FF: h_1(u) = \ldots = h_t(u) = 0\}$.  We say that $\mathcal{V}$ is cut out by the polynomials $h_1, \ldots, h_t$.
A variety $\mathcal{V}_1$ is a subvariety of $\mathcal{V}$ if $\mathcal{V}_1 \subset \mathcal{V}$.
\end{definition}

Given homogeneous polynomials $h_1,\ldots,h_t \in \RR[u]$, we use the notation $\mathcal{V}_\CC := \{ u \in \PP_\CC^n \,:\, h_i(u)=0 \,\,\hbox{for} \,\,i=1,\ldots,t \}$ for their projective
variety over the complex numbers, and $\mathcal{V}_\RR := \mathcal{V}_\CC \cap \PP_\RR^n$ for the
set of real points in $\mathcal{V}_\CC$.

\begin{definition}[Irreducibility]
A projective variety $\mathcal{V} \subseteq \PP^{n}_\CC$ is irreducible if it is not the union of two nonempty proper subvarieties.
\end{definition}

\begin{definition}[Dimension]
The dimension $\dim (\mathcal{V})$ of a projective variety $\mathcal{V}$ is
the largest number $d$ such that there is a chain
$
\emptyset \neq \mathcal{V}_0 \subsetneq \mathcal{V}_1 \subsetneq \cdots \subsetneq \mathcal{V}_d \subseteq \mathcal{V}
$
 of irreducible  subvarieties $\mathcal{V}_i$ of $\mathcal{V}$.

\end{definition}

The following result is standard and crucial to this paper:
\begin{theorem} \label{thm:intersections}
 Suppose $\mathcal{V}, \mathcal{W}\subseteq \PP^{n}_\CC$ are two irreducible projective varieties. If $\dim(\mathcal{V}) + \dim(\mathcal{W}) = n$ then $\mathcal{V}\cap \mathcal{W} \neq \emptyset$.
If $\dim(\mathcal{V}) + \dim(\mathcal{W}) > n$ then $\mathcal{V}$ intersects $\mathcal{W}$ at infinitely many points. These statements may fail over the field $\RR$ of real numbers.
\end{theorem}

For any $A\in \CC^{m\times (n+1)}$,
the linear space $\ker_\CC (A) = \{u \in \PP^{n}_\CC: Au = 0\}$ can be viewed as an irreducible projective variety. Its dimension is $n-{\rm rank}(A)$. This is one less than the dimension of $\ker_\CC (A)$ as a vector space. For example, suppose $A\in \RR^{7\times 9}$ has rank seven. Then as a projective variety, the dimension of  $\ker_\CC (A)$ equals one.  If $\mathcal{V}\subseteq \PP^8_\CC$ is a projective variety of dimension seven, then  $\ker_\CC (A)\cap \mathcal{V} \neq \emptyset$.

\begin{definition}[Degree]
The  $\degree (\mathcal{V})$ of a  projective variety $\mathcal{V} \subseteq \PP^{n}_\CC$ is the number of intersection points with a generic linear subspace of dimension $n- \dim(\mathcal{V})$.
\end{definition}

Given the homogeneous polynomials $h_i$ that cut out a projective variety $\mathcal{V}$,
we can compute ${\rm dim}(\mathcal{V})$, ${\rm degree}(\mathcal{V})$,
and irreducible components, using Gr\"obner bases~\cite{cox2007ideals}.

\section{Fundamental Matrices}
\label{sec:fundamental}

Following Section \ref{sec:epipolar},
a fundamental matrix is any real $3 \times 3$ matrix of rank two defined up to scale~\cite[Section 9.2.4]{hartley-zisserman-2003}.  We denote the set of
fundamental matrices~by
$$ \mathcal{F} \,\,:= \,\,\{ f \in \PP_\RR^8 \,:\, \rank(F) = 2 \}.$$
Recall that in our notation, the vector $f$ is
the concatenation of the rows of the matrix $F$.
Let $\mathcal{R}_1 := \{a \in \PP_\CC^8 \,:\, \rank(A) \leq 1 \}$ be the set of complex rank one matrices, and $\mathcal{R}_2 := \{ a\in \PP_\CC^8 \,:\, \rank(A) \leq 2 \}$ the set of complex matrices of rank
at most two, up to scale. These are irreducible projective varieties in $\PP_\CC^8$, with
$$\dim(\mathcal{R}_2) = 7, \,\,\degree(\mathcal{R}_2)=3
\quad \textup{ and } \quad \dim(\mathcal{R}_1) = 4, \,\,\degree(\mathcal{R}_1)=6.$$
Note that $\mathcal{R}_1 = \{ a \in \PP_\CC^8 \,:\, q_{ij}(A) = 0, \,\,1 \leq i,j \leq 3\}$ where the $q_{ij}$ are the $2 \times 2$ minors of a $3 \times 3$ matrix, and
$\mathcal{R}_2 = \{ a \in \PP_\CC^8 \,:\, \det(A) = 0 \}$.
On the other hand,
$\mathcal{F} = (\mathcal{R}_2 \backslash \mathcal{R}_1) \cap \PP_\RR^8$ is not a variety, but
only a {\em quasi-projective} variety over  $\RR$.

Suppose we are given $m$ homogenized point correspondences $\{ (x_i, y_i), \,\,i=1,\ldots,m\}
\subset \PP_\RR^2 \times \PP_\RR^2$. We represent these data by
  three matrices with $m$ rows:
  \begin{align} \label{3matrices}
X = \begin{bmatrix}
x_1^\top\\
\vdots\\
x_m^\top
\end{bmatrix} \in \RR^{m \times 3}, \
Y = \begin{bmatrix}
y_1^\top\\
\vdots\\
y_m^\top
\end{bmatrix} \in \RR^{m \times 3},\ \textup{and }
Z = \begin{bmatrix}
y_1^\top \otimes x_1^\top\\
\vdots\\
y_m^\top \otimes x_m^\top
\end{bmatrix} \in \RR^{m \times 9}.
\end{align}
This notation allows us to write the epipolar constraints (\ref{eq:epipolar})
as $Zf = 0$.
Hence a fundamental matrix $F$ exists for the
$m$ given point correspondences
  if and only if the linear subspace
 $\ker_\RR(Z)$ intersects the quasi-projective variety $\mathcal{F}$.
In symbols,
\begin{align} \label{eq:exists F}
\ker_\RR(Z) \cap \mathcal{F} \neq \emptyset.
\end{align}
We will characterize the conditions on
the $m \times 9$-matrix $Z$ under which~\eqref{eq:exists F} holds.

If $\rank(Z) = 9$, then $\ker_\RR(Z)\subset \PP_\RR^8$ is empty,
so there is no fundamental matrix.
Hence, we may assume $\rank(Z) \leq 8$. If $\rank(Z)=8$ then $\ker_\RR(Z)$ is a point $a$
in $\PP_\RR^8$, given by a $1$-dimensional space of matrices $A$.
  The following lemma is immediate.

\begin{lemma} \label{lem:rank=8forF}
If $\rank(Z)=8$ then $Z$ has a fundamental matrix if and only if $\ker_\RR(Z) = \{ a \}$
has the property that $\det(A) = 0$ and $q_{ij}(A) \neq 0$ for at least one of the $2 \times 2$ minors $q_{ij}$ that cut out $\mathcal{R}_1$. In that case, $A$ is the unique fundamental matrix for $Z$.
\end{lemma}

We illustrate Lemma~\ref{lem:rank=8forF} in the simplest possible situation.

\begin{example} \textup{[{$m=\rank(Z)=8$}]:} \label{ex:m_equals_eight}
Consider eight point correspondences such that
$Z \in \RR^{8 \times 9}$ has rank $8$. Let $Z^i$ denote the $8 {\times} 8$ matrix obtained
from $Z$ by
deleting the $i$th column. By Cramer's rule, the unique (up to scaling) point $a$ in
$\ker_\RR(Z)$ is given by $a^\top = ( (-1)^{i-1} \det(Z^i) )_{i=1}^{9}$.
Then $Z$ has a fundamental matrix
if and only~if
$$ \det(A) = \det \begin{bmatrix}
\det(Z^{1}) &  - \det(Z^{2}) &  \det(Z^{3})\\
-\det(Z^{4}) &   \det(Z^{5}) &  -\det(Z^{6})\\
 \det(Z^{7}) &  - \det(Z^{8}) & \det(Z^9)
\end{bmatrix} = 0$$
and at least one of the nine $2 \times 2$ minors of $A$ is not zero. These minors are
$$\det(Z^1)\det(Z^5)-\det(Z^2)\det(Z^4),\, \ldots\,,\, \det(Z^1)\det(Z^9)-\det(Z^6)\det(Z^8). $$
\end{example}

\begin{remark} \label{rmk:cramer vectors}
\textup{
Notice that in Example~\ref{ex:m_equals_eight}, the check in Lemma~\ref{lem:rank=8forF} was written directly in terms of $Z$, or equivalently, the point correspondences $\{ (x_i, y_i) \,:\, i=1,\ldots,8\}$. In fact, whenever $\rank(Z)=8$, for any value of $m$, an element in
$\ker_\RR(Z)$ is of the form $a_S^\top = ( (-1)^{i-1} \det(Z_S^i) )_{i=1}^{9}$ where $Z_S$ is a maximal submatrix of $Z$ with
eight rows, and all non-zero vectors $a_S$ are the same up to scale. Therefore, $Z$ has a fundamental matrix if and only if any non-zero $a_S$ passes the test in Example~\ref{ex:m_equals_eight}.
}
\end{remark}

We now examine what happens when $\rank(Z) \leq 7$. First, a few easy facts:

\begin{lemma} \label{lem:basic F facts}
\begin{enumerate}
\item If $\rank(Z) \leq 7$ then $\ker_\RR(Z) \cap \mathcal{R}_2 \neq \emptyset$.%
\item If $\rank(Z) = 7$, then either $\ker_\RR(Z) \subseteq \mathcal{R}_2$  or
$|\ker_\RR(Z) \cap \mathcal{R}_2| \leq 3$.
\item Suppose $A_1, \ldots, A_t$ is a vector space basis of $\ker_\RR(Z)$ and $M(u) = \sum_{i=1}^t A_iu_i$.
Then $\ker_\RR(Z) \subseteq \mathcal{R}_2$ if and only if $\det(M(u))$ is the zero polynomial.
\item Further, $\ker_\RR(Z) \subseteq \mathcal{R}_1$ if and only if $q_{ij}(M(u)) = 0$ for all $1 \leq i,j \leq 3$.
\end{enumerate}
\end{lemma}

\begin{proof}
(1) is well known \cite[Section 11.1.2]{hartley-zisserman-2003}  and holds since
the $3 \times 3$-determinant is a cubic polynomial. Fact (2) follows from (1) since a cubic has at most three real roots.
For (3), note that every element in $\ker_\RR(Z)$ is of the form $M(u)$ for some $u \in \RR^t$. Therefore, $\det(M(u))$ is the zero polynomial if and only if all matrices in $\ker_\RR(Z)$ have rank at most~$2$.
The last claim comes from the fact that the polynomials $q_{ij}$ cut out $\mathcal{R}_1$.
\end{proof}

Despite Lemma~\ref{lem:basic F facts} (1), even when $\rank(Z) \leq 7$,
 $Z$ may not have a fundamental matrix since
$\ker_\RR(Z)$ may be entirely inside $\mathcal{R}_1$. Here is such an example.

\begin{example} \label{ex:only rank ones}
Consider the following seven point correspondences:
$$
{
Y = \left[ \begin{array}{rrr}
-3 &\,\,\, 5 & \,\,\,1 \\
5 &\,\,\, -2 &\,\,\, \,1\\
 8 &\,\,\, -9 &\,\,\, \,1  \\
 11 &\,\,\, -16 &\,\,\, \,1 \\
 14 &\,\,\, -23 &\,\,\, \,1 \\
 17 &\,\,\, -30 &\,\,\, \,1 \\
 20 &\,\,\, -37 &\,\,\, \,1
 \end{array}
\right] \quad \textup{ and } \quad
X = \left[ \begin{array}{rrr}
10 &\,\,\, 4 &\,\,\, 1 \\
-7 &\,\,\, 0 &\,\,\, 1\\
-4 &\,\,\, 4 &\,\,\, 1\\
-7 &\,\,\, 1 &\,\,\, 1 \\
0 &\,\,\, -1 &\,\,\, 1 \\
1 &\,\,\, -8 &\,\,\, 1 \\
1 &\,\,\, -4 &\,\,\, 1
\end{array}
\right].
}
$$
Here $\rank(Z) = 7$ and $\ker_\RR(Z)$ is spanned by the rank one matrices
$$ A_1 =\left[ \begin{array}{rrr}
7 &\,\,\, 0 &\,\,\, -70\\
3 &\,\,\, 0 &\,\,\, -30\\
-29 &\,\,\, 0 &\,\,\, 290
\end{array} \right] \textup{ and }
A_2 = \left[ \begin{array}{rrr}
0 &\,\,\, 7 &\,\,\, -28\\
0 &\,\,\, 3 &\,\,\, -12\\
0 &\,\,\, -29 &\,\,\, 116
\end{array}
\right].$$
All nine  $2 \times 2$-minors of the following matrix are identically zero:
$$ u_1A_1 + u_2A_2 = 
\left[ \begin{array}{ccc}
7u_1 &\,\,\, 7u_2 &\,\,\, -70u_1-28u_2\\
3u_1 &\,\,\, 3u_2 &\,\,\, -30u_1 -12u_2\\
-29u_1 &\,\,\, -29u_2 &\,\,\, 290u_1+ 116u_2
\end{array}
\right] $$
We are in the situation of Lemma~\ref{lem:basic F facts} (4):
all matrices in $\ker_\RR(Z)$ have rank one.
 \end{example}

Let us now see what happens if the rank of $Z$ is small. One would expect that a large kernel  would make intersection with $\mathcal{F}$ easier. Indeed this is the case. To prove this, we start with a generalization of a result of Chum et al.~\cite{matas-et-al-2005}. Recall that a homography on $\PP^2_\RR$ is a $3\times3$ invertible matrix.

\begin{lemma}\label{lemma:chum}
If $m-2$ or more of the point correspondences are related by a homography $H$, i.e., for at least $m-2$ of the indices $i \in [m]$, $y_i \sim H x_i$, then there exists a fundamental matrix relating all the correspondences.
\end{lemma}
\begin{proof}
Recall that a fundamental matrix can be written in the form
\begin{align}
F = [b]_\times H
\end{align}
where $b$ is a non-zero vector in $\RR^3$ and $H$ is a real invertible matrix. Then the epipolar constraints can be re-written as
\begin{align}
& y_i^\top F x_i = 0, \forall i = 1, \hdots, m. \nonumber\\
\iff &y_i^\top [b]_\times Hx_i = 0, \forall i = 1, \hdots, m. \nonumber\\
\iff &y_i^\top[b \times H x_i] = 0, \forall i = 1, \hdots, m. \label{eq:before}\\
\iff &b^\top[y_i \times H x_i] = 0, \forall i = 1, \hdots, m.\label{eq:after}\\
\iff &b^\top\begin{bmatrix} \cdots & y_i \times H x_i & \cdots \end{bmatrix} = 0 \nonumber\\
\iff &\rank \begin{bmatrix} \cdots & y_i \times H x_i & \cdots \end{bmatrix} < 3\label{eq:rank}
\end{align}
The equivalence of~\eqref{eq:before} and~\eqref{eq:after} follows from the fact that $p^\top (q \times r) = -q^\top(p \times r)$. The matrix in~\eqref{eq:rank} is of size $3 \times m$. A sufficient condition for it to have rank less than 3 is for $m-2$ or more columns to be equal to zero. This is the case if $m-2$ or more of the pairs $(x_i,y_i)$ are related by the homography $H$.
\end{proof}

The observation about the scalar triple product and the resulting rank constraint 
has also been used by Kneip {\em et al.}~\cite{kneip-et-al} but only in the calibrated case. Using the above lemma we can now prove the following.

\begin{theorem}\label{thm:rank4}
If $\rank(Z) \le 4$, then $Z$ has a fundamental matrix.
\end{theorem}
\begin{proof}
For a set $\tau \subseteq [m]$, let $Z_\tau$ denote the submatrix of $Z$ whose row indices
are in $\tau$ (For completeness, $Z_{[m]\backslash \tau}$ is the submatrix of $Z$ whose row indices are not in $\tau$ and $Z_\emptyset := 0\in \RR^{1\times 3}$). It is straight forward to see that if $\rank(Z) = k$ then there exists a set $\tau \subseteq [m]$ of size $k$ such that $\rank(Z_\tau) = k$ and $Z$ has a fundamental matrix if and only if $Z_\tau$ has a fundamental matrix. Thus, we can restrict our discussion to the case $m \le 4$, with $Z$ having full row rank.

For $m=1$ and $m=2$, $H=I$, the identity matrix satisfies equation~\eqref{eq:rank}. For $m = 3$, one can always construct an invertible matrix $H$ such that $y_1 \sim H x_1$ which implies that $y_1 \times H x_1 = 0$ and equation~\eqref{eq:rank} is satisfied.

Let us now consider the case $m = 4$.  Since we assumed that $\rank(Z) = 4$, we must have  $\rank(X) \ge 2$ and $\rank(Y) \ge 2$. To see this, recall that each row of $Z$ is $z_i = y_i^\top \otimes x_i^\top$. Since $Y$ is a $4 \times 3$ matrix, $\rank(Y) \leq 3$, and so there is a
non-zero vector $\mu$ such that $\mu^\top Y = 0$, i.e.,
\begin{equation} \label{eq:coordinate-wise zero}
\sum_{i=1}^4 \mu_i y_{ij} = 0,\ j = 1,2,3.
\end{equation}

If $\rank(X) = 1$, we may assume that all rows of $X$ are scalar multiples of $x_1$. In fact, since the last coordinate of any $x_i$ is one, all rows are equal to $x_1$. Then $z_i = y_i^\top \otimes x_1^\top$, which implies that  
\begin{align} \label{eq:dependence}
\mu^\top Z = 0.
\end{align}
This means that $\rank(Z) < 4$ which is a contradiction. Therefore, $\rank(X) \geq 2$.

To see \eqref{eq:dependence}, note that the first coordinate of $\mu^\top Z$ is
\begin{align*}
\sum_i \mu_i y_{i1} x_{11}
                  = x_{11} \sum_i \mu_i y_{i1}
                   = 0
\end{align*}
by \eqref{eq:coordinate-wise zero}. Similarly all the other coordinates of $\mu^\top Z$ are also zero.  Exchanging the roles of $X$ and $Y$ we also get that $\rank(Y) \geq 2$.

If we can find two indices $i$ and $j$ such that the matrices $\begin{bmatrix} x_i & x_j\end{bmatrix}$ and $\begin{bmatrix} y_i & y_j\end{bmatrix}$ both have rank 2 then we can construct an invertible matrix $H$ such that $y_i \sim H x_i$ and $y_j \sim H x_j$ and that would be enough for~\eqref{eq:rank}.
Without loss of generality let us assume that the matrix $\begin{bmatrix} x_1 & x_2\end{bmatrix}$ is of rank 2, i.e., $x_1 \not\sim x_2$. If $\begin{bmatrix} y_1 & y_2\end{bmatrix}$ has rank 2 we are done. So let us assume that this is not the case and $y_2 \sim y_1$. Since $\rank(Y) \ge 2$, we can without loss of generality assume that $y_3 \not\sim y_1$. Since $x_1 \not\sim x_2$, either, $x_3 \not\sim x_1$ or $x_3 \not\sim x_2$.
In the former case, $i = 1, j = 3$ is the pair we want, otherwise $i = 2, j = 3$ is the pair we want.
\end{proof}

We next characterize when $\ker_\RR(Z)$ contains a rank one matrix.
\begin{lemma} \label{lem:sameer}
The intersection $\ker_\RR(Z) \cap \mathcal{R}_1 $ is non-empty if and only if
there exists $\tau \subseteq [m]$ such that the points in $\{x_i\}_{i\in\tau}$ and $\{y_i\}_{i \in [m]\setminus \tau}$
are each collinear in $\RR^2$.
\end{lemma}

\begin{proof}
The collinearity condition is equivalent to the linear algebra statement
\begin{align} \label{eq:set_tau}
\ker_\RR (X_\tau) \neq \{0\} \textup{ and  } \ker_\RR (Y_{[m]\backslash \tau}) \neq  \{0\}.
\end{align}
We claim that $\ker_\RR(Z) \cap \mathcal{R}_1 \neq \emptyset$ if and only if (\ref{eq:set_tau}) holds.
If $a\in  \ker_\RR(Z) \cap \mathcal{R}_1 $,  then $A = uv^\top$ for some $u,v\in \RR^3 \backslash \{0\}$.
For each $i$, we get $y^\top_i A x_i = (y_i^\top u)  (x_i^\top v) = 0$. Therefore, either $y_i^\top u =0$ or $  x_i^\top v = 0$.
Then $\tau := \{ i \in [m] \,:\, x_i^\top v =0 \}$ is a set of the desired form.
On the other hand, if there is a set $\tau \subseteq [m]$ such that
$\ker_\RR (X_\tau) \neq \{0\}$ and  $\ker_\RR (Y_{[m]\backslash \tau}) \neq  \{0\}$,
pick elements  $v \in$ ker$(X_\tau) \backslash \{0\}$ and
$u \in $ ker$(Y_{ [m] \setminus \tau}) \backslash \{0\}$, and set $A := uv^\top$.
For all $i$, we have $y^\top_i uv^\top x_i = 0$, and hence $a \in   \ker_\RR(Z) \cap \mathcal{R}_1 $.
\end{proof}

If $5\le \rank(Z) \leq 7$, then by Lemma~\ref{lem:basic F facts} (1) and Lemma~\ref{lem:sameer}, if
 no $\tau \subseteq [m]$ satisfies (\ref{eq:set_tau})
then $Z$ has a fundamental matrix. In order to get a complete characterization of when $Z$ has a fundamental matrix, we need to understand precisely when $\ker_\RR(Z)$ contains a rank two matrix. There are two cases:
\begin{center}
\begin{tabular}{ll}
either & (1) all matrices in $\ker_\RR(Z)$ have rank at most 2, i.e., $\ker_\RR(Z) \subseteq \mathcal{R}_2$\\
 or  & (2) $\ker_\RR(Z)$ contains an invertible matrix.
\end{tabular}
\end{center}

By Lemma~\ref{lem:basic F facts} (3),  case (1) is witnessed by  $\det(M(u)) = \det(\sum_{i=1}^t A_i u_i)$ being
the zero polynomial, where $\{ A_1, \ldots, A_t \}$ is a basis of $\ker_\RR(Z)$. Further, we can check if
$\ker_\RR(Z) \subseteq \mathcal{R}_1$ by using Lemma~\ref{lem:basic F facts} (4). If some $2 \times 2$ minor of $M(u)$ is not the zero polynomial, then $\ker_\RR(Z)$ contains a rank two matrix.
 Therefore, we can focus on case (2) in the rest of this section. In that case,
 $\det(M(u))$ is a non-zero polynomial.
We begin with a lemma that guarantees a rank two matrix in $\ker_{\RR}(Z)$.

\begin{lemma} \label{lem:bernd}
Let $L$ be a positive dimensional subspace in the projective space of real $3 \times 3$ matrices that contains a matrix of rank three. If  the determinant  restricted to $L$ is not a power of a linear form, then $L$ contains a real matrix of rank two.
\end{lemma}

\begin{proof} We first consider the case when $L$ is a projective line, i.e., $$L = \{A\mu + B\eta \,:\, \mu,\eta \in \RR \}$$  for some  $A,B \in \RR^{3 \times 3}$, with $B$ invertible. Then $B^{-1}L = \{ B^{-1}A\mu + I\eta \,:\,\mu,\eta \in \RR \}$ is an isomorphic image of $L$ and contains a matrix of rank two if and only if $L$ does.
 Hence we can assume $L = \{ M \mu - I \eta \,:\, \mu,\eta \in \RR\}$ for some $M \in \RR^{3 \times 3}$. The homogeneous cubic polynomial $\det(M \mu - I\eta)$ is not identically zero on $L$.
 When dehomogenized by setting $\mu=1$, it is the characteristic polynomial of $M$. Hence the three roots of $\det(M \mu - I \eta)=0$ in $\PP^1$ are $(\mu_1,\eta_1) \sim (1,\lambda_1), (\mu_2,\eta_2) \sim (1,\lambda_2)$ and $(\mu_3,\eta_3) \sim (1,\lambda_3)$ where $\lambda_1, \lambda_2, \lambda_3$ are the eigenvalues of $M$. At least one of these roots is real since $\det(M\mu - I\eta)$ is a cubic. Suppose $(\mu_1,\eta_1)$ is real. If $\rank(M\mu_1-I\eta_1) = \rank(M-I\lambda_1) = 2$, then $L$ contains a rank two matrix.  Otherwise, $\rank(M-I\lambda_1) = 1$. Then $\lambda_1$ is a double eigenvalue of $M$ and hence equals one of $\lambda_2$ or $\lambda_3$. Suppose $\lambda_1 = \lambda_2$. This implies that $(\mu_3,\eta_3)$ is a real root as well. If it is different from $(\mu_1,\eta_1)$, then it is a simple real root. Hence, $\rank(M\mu_3-I\eta_3)=2$, and $L$ has a rank two matrix.
So suppose $(\mu_1,\eta_1) \sim (\mu_2, \eta_2) \sim (\mu_3,\eta_3) \sim (1,\lambda)$ where $\lambda$ is the unique eigenvalue of $M$.
In that case, $\det(M\mu - I\eta) = \alpha \cdot (\eta-\lambda \mu)^3$ for some constant $\alpha$.
This finishes the case $\dim(L)=1$.

Now suppose $\dim(L) \geq 2$. If $\det$ restricted to $L$ is not a power of a homogeneous linear polynomial, then there exists a projective line $L'$ in $L$ such that $\det$ restricted to $L'$ is also not the power of a homogeneous linear polynomial.
The projective line $L'$ contains a matrix of rank two by the above argument.
\end{proof}

\begin{theorem} \label{thm:consequence for Z}
Suppose $\rank(Z) \leq 7$, $\ker_\RR(Z)$ contains a matrix of rank three and $M(u) = \sum_{i=1}^t A_i u_i$,
where $\{A_1, \ldots, A_t\}$ is a basis for $\ker_\RR(Z)$. If $\det(M(u))$ is not the third power of a
linear form then $\ker_\RR(Z)$ contains a $3 \times 3$ matrix of rank two.
\end{theorem}

\begin{example}
Consider the following seven point correspondences:
$$
Y = \left[ \begin{array}{rrr}
-2 &\,\,\, -3 & \,\,\,1 \\
-2 &\,\,\, -1 &\,\,\, \,1\\
 -2 &\,\,\, 2 &\,\,\, \,1  \\
 2 &\,\,\, -1 &\,\,\, \,1 \\
 2 &\,\,\, 0 &\,\,\, \,1 \\
 6 &\,\,\, -5 &\,\,\, \,1 \\
 7 &\,\,\, -6 &\,\,\, \,1
 \end{array}
\right]\quad  \textup{ and } \quad
X = \left[ \begin{array}{rrr}
2 &\,\,\, 0 &\,\,\, 1 \\
3 &\,\,\, -2 &\,\,\, 1\\
4 &\,\,\, -4 &\,\,\, 1\\
5 &\,\,\, -6 &\,\,\, 1 \\
6 &\,\,\, -8 &\,\,\, 1 \\
-3 &\,\,\, -2 &\,\,\, 1 \\
2 &\,\,\, -2 &\,\,\, 1
\end{array}
\right].
$$
Here $\rank(Z) = 7$ and $\ker_\RR(Z)$ is spanned by the following matrices
of rank $1$ and $3$:
$$ A_1 =\left[ \begin{array}{rrr}
2 &\,\,\, 1 &\,\,\, -4\\
2 &\,\,\, 1 &\,\,\, -4\\
-2 &\,\,\, -1 &\,\,\, 4
\end{array} \right] \quad \textup{ and } \quad
A_2 = \left[ \begin{array}{rrr}
0 &\,\,\, 5 &\,\,\, 6\\
-56 &\,\,\, -30 &\,\,\, 88\\
-272 &\,\,\, -152 &\,\,\, 484
\end{array}
\right]. $$
We find that $\det(M(u)) = \det(A_1 u_1 + A_2 u_2) =  96 (u_1 + 187 u_2) u_2^2$
 is not a power of a linear form. The combination
  $(-187)A_1+A_2$ is a rank two matrix in $\ker_\RR(Z)$.
\end{example}

Lemma~\ref{lem:bernd} is only a sufficient condition for the existence of a rank two matrix in $\ker_\RR(Z)$.
This means that when
$\det(M(u))$ is indeed the cubic of a linear form, $\ker_\RR(Z)$ may or may not
contain a rank two matrix, as the following examples show.

\begin{example} \label{ex:cubic linear form examples}
(i) Consider the seven point correspondences given by
$$
Y = \left[ \begin{array}{ccc}
0 &\,\,\, 1 & \,\,\,1 \\
1 &\,\,\, 0 &\,\,\, \,1\\
 2 &\,\,\, 5 &\,\,\, \,1  \\
 3 &\,\,\, \frac{-5}{12} &\,\,\, \,1 \\
 4 &\,\,\, 7 &\,\,\, \,1 \\
 5 &\,\,\, \frac{-11}{8} &\,\,\, \,1 \\
 6 &\,\,\, 9 &\,\,\, \,1
 \end{array}
\right] \quad \textup{ and } \quad
X = \left[ \begin{array}{ccc}
\frac{-1}{5} &\,\,\, -1 &\,\,\, 1 \\
-1 &\,\,\, -7 &\,\,\, 1\\
\frac{-1}{2} &\,\,\, 0 &\,\,\, 1\\
-2 &\,\,\, -12 &\,\,\, 1 \\
\frac{-57}{4} &\,\,\, 8 &\,\,\, 1 \\
2 &\,\,\, 8 &\,\,\, 1 \\
0 &\,\,\, \frac{-1}{9} &\,\,\, 1
\end{array}
\right].
$$
Again, $\rank(Z)=7$ and $\ker_\RR(Z)$ is spanned by the $3 \times 3$ identity matrix $I$ and
$A_2 = \left[ \begin{array}{ccc} 0 & 1 & 2 \\ 5 & 4 & -2 \\ -15 & 3 & 11
\end{array} \right]$ where $\rank(A_2)=3$. Check that $\det(Iu_1 + A_2u_2) = (u_1 + 5u_2)^3$. The
unique (up to scale) rank-deficient matrix in $\ker_\RR(Z)$ is the rank one matrix
$$A_2-5I = \left[ \begin{array}{ccc}
-5 & 1 & 2 \\ 5 & -1 & -2 \\ -15 & 3 & 6 \end{array} \right].$$

(ii) Now consider the seven point correspondences given by
$$
Y = \left[ \begin{array}{ccc}
1 &\,\,\, 0 & \,\,\,1 \\
\f{1}{3} &\,\,\, 0 &\,\,\, \,1\\
 \f{1}{3} &\,\,\, -1 &\,\,\, \,1  \\
1 &\,\,\, -1 &\,\,\, \,1 \\
 \f{1}{2} &\,\,\, -1 &\,\,\, \,1 \\
 4 &\,\,\, -2 &\,\,\, \,1 \\
 2 &\,\,\, -2 &\,\,\, \,1
 \end{array}
\right] \quad \textup{ and } \quad
X = \left[ \begin{array}{ccc}
-1 &\,\,\, 0 &\,\,\, 1 \\
-3 &\,\,\, 0 &\,\,\, 1\\
6 &\,\,\, 3 &\,\,\, 1\\
0 &\,\,\, 1 &\,\,\, 1 \\
2 &\,\,\, 2 &\,\,\, 1 \\
0 &\,\,\, \f{1}{2}&\,\,\, 1 \\
\f{1}{2} &\,\,\, 1 &\,\,\, 1
\end{array}
\right].
$$
Here ${\rm ker}_\RR (Z)$ is spanned by the matrices
$A_1 =
 \begin{bmatrix} 1 & 0 & 0 \\ 0 & 1 & 0 \\ 0 & 0 & 1\end{bmatrix}
$ and $A_2 = \begin{bmatrix} 0 & 1 & 0 \\ 0 & 0 & 1 \\ 0 & 0 & 0\end{bmatrix}$, and
$\det(A_1 u_1 +A_2 u_2) = u_1^3$. Here, $A_2$ is a fundamental matrix of $Z$.
\end{example}

To finish, we need to be able to decide if $\ker_{\RR}(Z)$ contains a rank two matrix when $\det$ restricted to $\ker_{\RR}(Z)$ is of the form $(b^\top u)^3$ for a non-zero vector $b$.
Here we write $\ker_\RR(Z)$ as the space of matrices
$M(u) = \sum_{i=1}^t A_i u_i$.
The rank deficient matrices in $\ker_\RR(Z)$ correspond to solutions $u$ of $\det(M(u))=0$.
Since $\det(M(u)) = (b^\top u)^3$, for any $M(u)$ in our situation, the rank-deficient ones are precisely those in the set
\begin{align}
\mathcal{M} := \left\{ M(u) : u \in b^\perp \right\}
\quad \hbox{where} \,\,\, b^\perp := \{u \in \RR^t\ : b^\top u = 0\}.
\end{align}

\begin{remark}\label{rmk:projector}
The hyperplane $b^\perp$ consists of all vectors $\,u - \frac{b^\top u}{b^\top b} b \,$ where $ u \in \RR^t$.
\end{remark}

Therefore, $\mathcal{M} = \left\{ M\bigl(u - \frac{b^\top u}{b^\top b} b \bigr) : u \in \RR^{t} \right\}$. Now, recall that for a $3 \times 3$ matrix to be of rank 2, at least one of its $2 \times 2$ minors $q_{ij}$ must be non-zero. Thus we have:

\begin{lemma}\label{lem:cubic case}
If $\det$ restricted to $\ker_{\RR}(Z)$ is of the form $(b^\top u)^3$ for a non-zero vector $b$ then $\ker_{\RR}(Z) \cap \mathcal{F} = \emptyset$ if and only if $ q_{ij}\!\left(M\bigl(u - \frac{b^\top u}{b^\top b} b\bigr)\right) = 0$
for all $1 \leq i,j \leq 3$.
\end{lemma}

\begin{example}
We illustrate Lemma \ref{lem:cubic case} for Example~\ref{ex:cubic linear form examples} (ii).
Recall that $M(u) = A_1 u_1 +A_2 u_2$ and $\det(M(u)) = u_1^3$.
Hence $b = (1, 0)^\top$ and
   $u - \frac{b^\top u}{b^\top b} b = (0, u_2)^\top$.
Thus $M\bigl(u - \frac{b^\top u}{b^\top b} b \bigr) =  A_2 u_2$
has rank $2$ for $u_2 \not=0$, and
 $A_2$ is the fundamental matrix.
\end{example}

We now have a complete procedure
to check if a given set of point correspondences $\{(x_i, y_i) \,:\, i=1,\ldots,m \}$ has a fundamental matrix. Our algorithm is displayed below.
Its input is the matrix $Z$ that is derived from the points $x_i$ and $y_i$ as in (\ref{3matrices}).

\begin{algorithm}
\caption{Existence check for the fundamental matrix}
\begin{algorithmic}
\Require{$Z \in \RR^{m \times 9}$}
\If {$\det\left(Z^\top Z\right) \neq 0$}
\State {Return False.}
\EndIf

 \State{Compute a basis $A_1, \ldots, A_t \in \RR^{3 \times 3}$ for $\ker_\RR(Z)$}
 \State{$M(u) = \sum_{i=1}^t A_iu_i$}
 \State {$p(u) = \det\left(M(u)\right)$}
 \If {$p(u) = 0$}
    \If{$\forall i,j\ q_{ij}\! \left(M\left(u\right)\right) = 0$}
    \State{Return False.}
    \EndIf
 \Else
  	\If {$p(u)$ is of the form $(b^\top u)^3$ and
	$\forall i,j\ q_{ij}\! \left(M\bigl(u - \frac{b^\top u}{b^\top b} b\bigr)\right) = 0$}
    	\State{Return False.}

    \EndIf
  \EndIf
\State{Return True}
\end{algorithmic}
\label{algo:existenceOfF}
\end{algorithm}

We offer several comments.  Algorithm 1 has no restriction on $m$ or the rank of $Z$.
The case $\rank(Z)=9$ is excluded at the start.
If $\rank(Z) = 8$, then $M(u_1) = A u_1$  where $a$ is a generator of the one dimensional vector space $\ker_\RR(Z)$. Now, $Z$ has a fundamental matrix if and only if $\rank(A) = 2$. In the steps of the algorithm, $p(u_1) =\det(M(u_1)) = \det(A u_1) = [(\det(A))^{1/3} u_1]^{3}$.  So, if $p(u_1) = 0$ then $\det(A) = 0$ and $Z$ has a fundamental matrix if and only if $A$ has a non-zero $2 \times 2$ minor. If $p(u_1) \neq 0$ then it is a cube of a linear form, with $b = (\det(A))^{1/3}$.
Here, $q_{11}\bigl(M(u_1 - \frac{b u_1}{b^2} b)\bigr) = q_{11}(M(0)) = 0$, so there is no
  fundamental matrix for $Z$. The remaining cases $\rank(Z) \le 7$ follow from Lemma~\ref{lem:basic F facts} and Theorem~\ref{thm:consequence for Z}.

Algorithm 1 can be expressed in terms of polynomials in $X$ and $Y$ since
the vectors in a basis of $\ker_\RR(Z)$ have coordinates that are polynomials in $X$ and $Y$.
This uses the notion of Cramer vectors which we alluded to in Remark~\ref{rmk:cramer vectors}.
We already saw this in action for the $\rank(Z) = 8$ case. The complexity of calculating a basis for $\ker_\RR(Z)$ is polynomial in $m$  and the size of the correspondences. All other steps run in constant time. This gives:

\begin{corollary} Given $m$ point correspondences $\{(x_i,y_i) \,:\, i=1,\ldots, m \}$, it
can be decided in polynomial time whether they have a fundamental matrix.
\end{corollary}

Note that could use Theorem~\ref{thm:rank4} to exit early in Algorithm 1 if $\rank(Z) \le 4$. 
This would not affect the overall running time or correctness of the algorithm.

\section{The Essential Matrix}
\label{sec:essential}
We now turn our attention to essential matrices.
Let $P_1$ and $P_2$ be two non-coincident cameras.
We write $x_i \sim P_1 w_i$ and $y_i \sim P_2 w_i$ for the images of the world point $w_i$. Let us now assume that the calibration matrices $K_1$ and $K_2$ for $P_1$ and $P_2$ are known. Then recall from the introduction that an essential matrix $E$ of the two normalized cameras $K_1^{-1}P_1$ and $K_2^{-1}P_2$ satisfies the epipolar constraints $\wh{y}_i^\top E  \wh{x}_i \,=\, 0$
where
$\widehat{x}_i = K_1^{-1} x_i$ and $\widehat{y}_i = K_2^{-1} y_i$ are the normalized image coordinates.

The set of essential matrices is the set of real $3 \times 3$ matrices of rank two with two equal (non-zero) singular values~\cite{maybank-faugeras}. We denote the set of essential matrices by
\begin{equation}
\label{eq:singvalues}
	\mathcal{E}_\RR = \{e \in \PP^8_\RR : \sigma_1(E) = \sigma_2(E)
	\,\,\,{\rm and} \,\,\, \sigma_3(E) = 0\},
\end{equation}
where $\sigma_i(E)$ denotes the $i^{\rm{th}}$ singular value of the matrix $E$. This in particular implies that
$\mathcal{E}_\RR$ is contained entirely in $\mathcal{R}_2 \setminus \mathcal{R}_1$.

Demazure~\cite{demazure} showed
\begin{align} \label{E_describe}
\mathcal{E}_\RR \,= \bigl\{e \in \PP^8_\RR : p_j(e) = 0 \,\,\, {\rm for} \,\,\, j = 1, \hdots, 10 \bigr\},
\end{align}
where the $p_j$'s are homogeneous polynomials of degree three defined as
\begin{align}
\begin{bmatrix}
p_1 & p_2 & p_3 \\
p_4 & p_5 & p_6\\
p_7 & p_8 & p_9
\end{bmatrix} &:=  2 EE^\top E - {\rm Tr}(EE^\top) E,  \label{eq:tracecubics} \textup{ and }\\
p_{10} &:= \det(E). \label{eq:det}
\end{align}

Passing to the common complex roots of the cubics $p_1, \ldots, p_{10}$, we get
\begin{align}
\mathcal{E}_\CC := \{e \in \PP^8_\CC : p_j(e) = 0, \forall j = 1, \hdots, 10\}.
\end{align}
This is an irreducible projective variety
with $\dim(\mathcal{E}_\CC) = 5$ and $\degree(\mathcal{E}_\CC) = 10$~\cite{demazure}.
We are primarily interested in the corresponding
real projective variety $\mathcal{E}_\RR \subset \PP^8_\RR$.

\begin{remark}
It is also possible to see from the Demazure cubics that $\mathcal{E}_\RR  \cap \mathcal{R}_1 = \emptyset$.
A Gr\"obner basis computation shows that the nine sum of squares polynomials $(E_{i1}^2+E_{i2}^2+E_{i3}^2)(E_{1j}^2+E_{2j}^2+E_{3j}^2)$ vanish on the intersection. The vanishing of
one such polynomial implies that either the $i$th row or $j$th column of $E$ is zero, and the vanishing of all nine polynomials imply that all entries of $E$ must be zero.
\end{remark}

Our data consists of  $m$ point correspondences $\{ (x_i, y_i), \,\,i=1,\ldots,m\}$.
We now assume that $x_i$ and $y_i$ are normalized image coordinates.
We represent the correspondences by the matrix $Z$ in \eqref{3matrices}.
As in the uncalibrated case,  we can write the epipolar constraints  as $Ze = 0$.
 Thus  $Z$ has an essential matrix if and only if
\begin{align} \label{eq:exists E}
\ker_\RR(Z) \,\cap\, \mathcal{E}_\RR\,\, \neq \,\, \emptyset.
\end{align}
Hence the existence of an essential matrix for the given correspondences
is equivalent to the intersection of a subspace with a fixed real projective variety being non-empty.

If $\rank(Z) = 9$ then $Z$ has no essential matrix.
The case $\rank(Z) = 8$ is easy, since $\ker_\RR(Z)$ is just a point in $\mathbb{P}_\RR^8$.
If an essential matrix exists then it is unique:

\begin{lemma} \label{lem:E_for_rank8}
If $\rank(Z)=8$ then $Z$ has an essential matrix if and only if the unique point $a \in \ker_\RR(Z)$,
found from $Z$ by Cramer's rule, satisfies $p_1(a) = \cdots = p_{10}(a) = 0$.
\end{lemma}

We now consider what happens for $\rank(Z) \leq 7$.  First, we have the following analog of Lemma~\ref{lemma:chum} whose proof is exactly as before.

\begin{lemma}
If $m-2$ or more of the point correspondences are related by a matrix $R \in SO(3)$, i.e., for at least $m-2$ of the indices $i \in [m]$, $y_i \sim R x_i$, then there exists an essential matrix relating all the correspondences.
\end{lemma}

The following is then immediate.

\begin{corollary}
If $\rank(Z) \le 3$, then $Z$ has an essential matrix.
\end{corollary}

Here we use that fact that given a pair of correspondences $(x,y)$, there is always a rotation matrix $R$ such that $y \sim Rx$.

Given the dimension and degree of $\mathcal{E}_\CC$,
Theorem~\ref{thm:intersections} implies
the following basic facts about the
complex version of~\eqref{eq:exists E}.

\begin{lemma} \label{lem:E_for_rank<=5}
\begin{enumerate}
\item If $\rank(Z) = 6$ or $7$ then $\ker_\CC(Z) \cap \mathcal{E}_\CC = \emptyset$ for generic data~$Z$.
\item If $\rank(Z) = 5$ then $\ker_\CC(Z) \cap \mathcal{E}_\CC \neq \emptyset$, and this intersection
consists of precisely  ten (complex) points when the given data $Z$ are generic.
\item If $\rank(Z)\leq 4$ then $\ker_\CC(Z) \cap \mathcal{E}_\CC$ is an infinite set.
For generic $Z$,  this intersection is an
irreducible variety of dimension $5-\rank(Z)$ and degree~$10$.
\end{enumerate}
\end{lemma}

For an essential matrix of $Z$ to exist we need a real point in $\ker_\CC(Z) \cap \mathcal{E}_\CC$. Therefore, for all ranks of $Z$,
the existence of an essential matrix for $Z$ is  the delicate question  of existence of real solutions for polynomial equations.  Lemma~\ref{lem:E_for_rank8} takes care of the situation in which $\rank(Z)=8$. However,
certifying the existence of a real point in $\ker_\CC(Z) \cap \mathcal{E}_\CC$ is
a considerably harder problem when $\rank(Z) \leq 7$. This, in theory, can be characterized by the {\em real Nullstellensatz} \cite{MarshallBook} and checked degree by degree via {\em semidefinite programming}. However, this is a case by case computation,
rather than a formula in terms of $Z$.

Before going further, one might wonder if there is, in fact, always a real point in $\ker_\CC(Z) \cap \mathcal{E}_\CC$.  We now show that
sometimes, no point in
$\ker_\CC(Z) \cap \mathcal{E}_\CC$ is real.

\begin{example} \label{ex:all complex Es}
We verified using \verb+Maple+ that the following set of five point correspondences has no essential matrix. All 10 points in
$\ker_\CC(Z) \cap \mathcal{E}_\CC$ are complex.
\begin{align}
Y = \left [\begin{array}{rrr}
2 & 0 & 1 \\
5 & 4 & 1 \\
 9 & 6 & 1 \\
 2 & 5 & 1 \\
 1 & 4 & 1 \\
\end{array}\right],\quad
X = \left[\begin{array}{rrr}
3 & 0 & 1 \\
9 & 1 & 1 \\
1 & 2 & 1 \\
8 & 8 & 1 \\
4 & 8 & 1
\end{array}\right]
\end{align}
\end{example}

By Lemmas~\ref{lem:E_for_rank8} and \ref{lem:E_for_rank<=5}, it remains to settle the cases of $\rank(Z)=6 \textup{ and } 7$.  The rest of the paper will be devoted to these two cases.

For $\rank(Z) = 7$, we will establish conditions on $Z$ such that  $\ker_\RR(Z) \cap \mathcal{E}_\RR \not= \emptyset $.  For $\rank(Z) = 6$, we will have to settle for characterizing the data $Z$ that satisfy $\ker_\CC(Z) \cap \mathcal{E}_\CC \not= \emptyset $. This is certainly necessary for an essential matrix of $Z$ to exist, and, as we
shall see, already involves non-trivial mathematics. These ranks can be dealt with via a tool from algebraic geometry that we explain next.

\subsection{The Chow form of a projective variety}
Let $\mathcal{V} \subset \PP_\CC^n$ be an irreducible projective variety of dimension $d$ and degree $\delta$. Recall that a generic
subspace $L \subset \PP_\CC^n$ of dimension $n-d$ will intersect $\mathcal{V}$ in $\delta$ points up to multiplicity, but if
$\dim(L) < n-d$ then $L$ will usually not intersect $\mathcal{V}$. We consider all subspaces of dimension exactly $n-d-1$ that
do intersect $\mathcal{V}$. To be concrete, let us express an $(n-d-1)$-dimensional subspace $L$ as $\ker_\CC(A)$ for some $A \in \CC^{(d+1) \times (n+1)}$. Then there is a single irreducible polynomial in the entries of $A$ called the {\em Chow form} of $\mathcal{V}$, denoted as $\textup{Ch}_{\mathcal{V}}(A)$, that defines the set of all $(n-d-1)$-dimensional subspaces $L$ that intersect $\mathcal{V}$.
More explicity, for a scalar matrix $A \in  \CC^{(d+1) \times (n+1)}$, ${\rm Ch}_\mathcal{V}(A) = 0$ if and only if
$\ker_\CC(A) \cap \mathcal{V} \neq \emptyset$. Hence ${\rm Ch}_\mathcal{V}$ certifies exactly when an $(n-d-1)$-dimensional subspace
intersects $\mathcal{V}$.
 Further, ${\rm Ch}_{\mathcal{V}}(A)$ is a homogeneous polynomial of degree $(d+1)\delta$ in the entries of $A$.
The computation of a Chow form is a problem in {\em elimination theory} \cite[Chapter 3]{cox2007ideals} and we illustrate the method in Example~\ref{ex:twisted cubic chow form}.
An accessible introduction to Chow forms can be found in \cite{dalbec-sturmfels}, with
algorithms for computing them in Section 3.1. For an in-depth treatment see \cite[Section 3.2 B]{gelfand-et-al}.

\begin{example} \cite[1.2]{dalbec-sturmfels} \label{ex:twisted cubic chow form}
Consider the cubic curve in $3$-space given parametrically as
$$\mathcal{C} \,\,=\,\, \{ (\lambda^3\,,\,\lambda^2 \mu\,,\,\lambda \mu^2\,,\,\mu^3) \in \PP_\CC^3 \,:\, \lambda, \mu \in \CC\}.$$
Here $n=3, d=1$ and $\delta=3$.
The Chow form ${\rm Ch}_{\mathcal{C}}$ characterizes the set of all lines
   $L \subset \PP_\CC^3$ that intersect $\mathcal{C}$.   Writing $L = \ker_\CC \begin{bmatrix} a_{11} & a_{12} & a_{13} & a_{14}\\ a_{21} & a_{22} & a_{23} & a_{24} \end{bmatrix}$,
\begin{equation}
\label{eq:Bezout}
  {\rm Ch}_{\mathcal{C}}(A) \,\, = \,\,
\det \begin{bmatrix}
a_{11} a_{22} - a_{12} a_{21} \ & \ a_{11} a_{23} - a_{13} a_{21} \ & \  a_{11} a_{24} - a_{14} a_{21} \\
a_{11} a_{23} - a_{13} a_{21} \ & \ \substack{a_{11} a_{24} - a_{14} a_{21}\\
+ a_{12} a_{23} - a_{13} a_{22}} \ & \  a_{12} a_{24} - a_{14} a_{22} \\
a_{11} a_{24} - a_{14} a_{21} \ & \ a_{12} a_{24} - a_{14} a_{22}  \ & \  a_{13} a_{24} - a_{14} a_{23} \\
\end{bmatrix} .
\end{equation}
This has degree $6 = 2 \cdot 3 = (d+1) \cdot \delta$ as a polynomial in the entries of $A$.
We have ${\rm Ch}_{\mathcal{C}}(A) = 0$ if and only if
 $L \cap \mathcal{C} \neq \emptyset$ if and only if there exists $(\lambda, \mu) \neq (0,0)$ with
$$a_{11} \lambda^3 + a_{12}\lambda^2 \mu + a_{13}\lambda \mu^2 + a_{14} \mu^3
\,\, =\,\, a_{21} \lambda^3 + a_{22}\lambda^2 \mu + a_{23}\lambda \mu^2 + a_{24} \mu^3
\,\,=\,\, 0.$$
These two binary cubics have a common root if and only if their {\em resultant} vanishes. Thus in this example,
${\rm Ch}_{\mathcal{C}}$ shown in \eqref{eq:Bezout} is given by the B\'ezout formula for the
resultant of two binary cubics \cite[\S 4.1]{bernd-cbmsbook}
\end{example}

The resultant of two binary cubics mentioned above carries further information that will be useful for us later.
The two cubics share a unique common root if and only if the {\em B\'ezout matrix} shown in \eqref{eq:Bezout} (whose determinant is
${\rm Ch}_{\mathcal{C}}(A)$) has rank exactly two \cite[pp. 42]{heinig-rost}. In this case, this unique common root is real. Otherwise, the two cubics (assuming they are not scalar multiples of each other), share two roots and hence their greatest common divisor (gcd) is a quadratic polynomial. The roots of this quadratic are the two common roots and they are real or complex depending on sign of the discriminant of the quadratic gcd.

Note that we can also determine whether $\mathcal{E}_\CC \cap \ker_\CC(Z) \neq \emptyset$ by computing a Gr\"obner basis of the set of Demazure polynomials and the epipolar constraints. By Hilbert's Nullstellensatz, the intersection is empty if and only if the Gr\"obner basis contains  $1$. However, this is a blackbox that needs to be invoked for each $Z$. The Chow form on the other hand is a fixed polynomial, depending on
$\mathcal{E}_\CC$, that just needs to be computed once.  Both approaches are based on elimination theory.

\subsection{$\rank(Z) = 6$}

Recall that $\mathcal{E}_\CC \subset \PP_\CC^8$ is an irreducible
variety  of dimension $5$ and degree $10$.
Hence, its Chow form ${\rm Ch}_{\mathcal{E}_\CC}$ is a polynomial of degree $60=6 \times 10$ in the entries of a
$6 \times 9$ matrix $A$.
 This polynomial tells precisely when $\ker_\CC(Z) \cap \mathcal{E}_\CC \neq \emptyset$.

\begin{theorem} \label{thm:E_for_rank6}
Suppose $\rank(Z) = 6$, and
let $Z'$ be any $6\times 9$ submatrix of $Z$ with $\rank(Z') = 6$.
Then $\,\ker_\CC(Z) \cap \mathcal{E}_\CC \neq \emptyset\,$ if and only if
$\,{\rm Ch}_{\mathcal{E}_\CC}(Z') = 0$.
\end{theorem}

\begin{proof}
Since $\rank(Z)=6$, we have
$\ker_\CC(Z) = \ker_\CC(Z')$ for any $6 \times 9$ submatrix $Z'$ of rank six. We saw earlier that
${\rm Ch}_{\mathcal{E}_\CC}(Z') \neq 0$ if and only if $\ker_\CC(Z') \cap \mathcal{E}_\CC = \emptyset$.
\end{proof}

The Chow form ${\rm Ch}_{\mathcal{E}_\CC}$ is a polynomial of degree $60$ in the
$54$ entries of the matrix $Z$. It has total degree $120$ when expressed in
the coordinates of the points $x_i$ and $ y_i$.
So, this is a huge polynomial. It would be desirable to find a compact determinantal
representation as that in Example~\ref{ex:twisted cubic chow form}, but presently
we do not know such a formula. 

Nevertheless, we can use Gr\"obner basis methods to
compute specializations of ${\rm Ch}_{\mathcal{E}_\CC}$
to few unknowns. For instance, when the given
data $Z$ depends on one or two parameters, we may wish to identify
all parameter values for which an essential matrix exists.
The following example illustrates that application of the Chow form.

\begin{example}
Suppose we are given the following six point correspondences:
\begin{align}
Y = \left [\begin{array}{rrr}
 7 & u & 1\\
 u & 3 &1\\
1 & u &1 \\
 1 & 2 &1\\
2 &1 &1\\
0 & 0 &1
\end{array}\right],\quad
X = \left[\begin{array}{rrr}
0 & 0 &1 \\
1 &3 &1\\
3 & 1 &1\\
u & 3 &1\\
4 & u &1\\
u & 5 &1
\end{array}\right]
\end{align}
Here $u$ is a parameter. Each row of the
corresponding $6 \times 9$-matrix $Z$
depends linearly on $u$. We seek to find all values of $u$ such
that $Z$ has an essential matrix.

The specialized Chow form ${\rm Ch}_{\mathcal{E}_\CC}(Z)$
is a polynomial of degree $60$ in $u$. It equals
\begin{tiny}
  $$ \begin{matrix}
  43834375\,{u}^{60}-4638778750\,{u}^{59}+214224546125\,{u}^{58}-5531342198800\,{u}^{57}+
81821365650850\,{u}^{56} \\
-502421104086298\,{u}^{55}-5192732464033792\,{u}^{54}+
146734050586490146\,{u}^{53}-1303188047383685288\,{u}^{52}\\
+\,5411546095695094\,{u}^{51}+
126933034857782682824\,{u}^{50}-1402930174630501908410\,{u}^{49} \\
+4805512868239289501170
\,{u}^{48}+54867304381377375967458\,{u}^{47}-946256866425258876320166\,{u}^{46} \\
+ 6828426599313060876754686\,{u}^{45}-14450347209949489264602097\,{u}^{44}-
283933191518660124999399676\,{u}^{43} \\
+4047868311171449853710938821{u}^{42} {-}
24993281401965265098244703998{u}^{41} {+} 15587695078296602756618405320{u}^{40} \\
+
1137047767757534377636418813724\,{u}^{39}-10221999792035763131237695639122\,{u}^{38} \\
+ 32470234671439206982327689429148\,{u}^{37}+146425390395530050950765067367816\,{u}^{36} \\
- 2155690692554728936349965416515780\,{u}^{35}+10321568760300476059189682724192744\,{u}^{
34} \\
-10122800683736158018034953038141540\,{u}^{33}-174813509014464010527815531126605212
\,{u}^{32} \\
+1206087628271760015231955008103896036\,{u}^{31}-
3436313795747065945575941840957490076\,{u}^{30} \\
-1363676250453653423248357040906303140\,
{u}^{29}+55723992159307840885015261004346771669\,{u}^{28} \\
-
251215280255927508990611099884445351606\,{u}^{27}+
544358623829097341051489976774086071975\,{u}^{26}\\
+
14464241162693808080892004258636405076\,{u}^{25}-
4558449315974650797311321984710224631310\,{u}^{24}\\
+
17812787901602232003024937224176380416478\,{u}^{23}-
38397784240820238777284236938586287303108\,{u}^{22} \\
+ 43276773504778039106632494947326914780890\,{u}^{21}+
15276370159804502922351511587351667789056\,{u}^{20} \\
- 155244069230397477917662963858164136786434\,{u}^{19}+
229768326200857711686707816150873030008800\,{u}^{18} \\ +
211532036240929904169570577258077028799230\,{u}^{17}-
1889026744395570650381887483449664405667446\,{u}^{16} \\ +
5471833600393127801394443745385356177011802\,{u}^{15}-
11088767784865709214849131148497436984866334\,{u}^{14} \\ +
17979801985915093648219684662006298683132934\,{u}^{13}-
24545309594399326756168652083447288177606955\,{u}^{12} \\ +
28845541790643822200455517412186608641540816\,{u}^{11}-
29403994066308211264842586547936962131413241\,{u}^{10} \\ +
25967381331061345555920067081171263369815898\,{u}^{9}-
19732702778911021713799125233709800780863612\,{u}^{8} \\ +
12764162704811266367782020135497648044854016\,{u}^{7}-
6927125308120500599068517401329046792275082\,{u}^{6} \\ +
3094196986553177513668480578495229924251400\,{u}^{5}-
1107994372586109623425960280114395072430496\,{u}^{4} \\ +
306014985831913575104245493823780969019440\,{u}^{3}-
61264148720465639197895147611951283785808\,{u}^{2} \\ +
7923839113658896236409956717070751269760\,u-497978312038181379674839759765953720000.
  \end{matrix}
$$
\end{tiny}
This polynomial has $60$ distinct roots in $\CC$. Of these $60$ roots, precisely $14$ are real:
$$
\begin{small}
\begin{matrix}
-8.89730484, -8.11776207, -3.46539765, 1.14351885, 1.41344366, 1.45874375, 1.53593853, \\
1.57524505, \,3.44472181,\, 4.64395895, \,5.21709064, \,7.35719977, \,11.9592629, \, 19.0073381.
\end{matrix}
\end{small}
$$
For each of these $14$ values of $u$, the intersection
$\ker_\CC(Z) \cap \mathcal{E}_\CC$ consists of a unique point $e$ with real coordinates.
A lexicographic Gr\"obner basis reveals these coordinates as elements of $\QQ[u]$, so
we get a formula in terms of $u$ for the $14$ essential matrices~$E$.
\end{example}

Despite the large size of ${\rm Ch}_{\mathcal{E}_\CC}$, a formula for it 
 can be computed offline, once and for all, in a single preprocessing step.
Thus, from the point of view of computational complexity,
we can compute ${\rm Ch}_{\mathcal{E}_\CC}$ in constant time. 
For any given data set with $\rank(Z)=6$, we can then plug a suitable
$6 \times 9$-submatrix $Z'$ into that formula to decide whether
 ${\rm Ch}_{{\mathcal{E}_\CC}}(Z') = 0$, and hence whether
 $\ker_\CC(Z) \cap \mathcal{E}_\CC \neq \emptyset$, by Theorem~\ref{thm:E_for_rank6}.
This furnishes an algorithm that runs in polynomial time in $m$ for
deciding whether a given $Z$ has an essential matrix over~$\CC$.

\subsection{$\rank(Z) = 7$} \label{subsec:rank_7}
Also in this case we can use the Chow form of $\mathcal{E}_\CC$ to
characterize when $\ker_\CC(Z) \cap \mathcal{E}_\CC \neq \emptyset$.
Pick a $7 \times 9$ submatrix $Z'$ of $Z$ such that $\ker_\CC(Z') = \ker_\CC(Z)$. Suppose $\ker_\CC(Z') \cap  \mathcal{E}_\CC \neq \emptyset$. Then $\ker_\CC(UZ') \cap  \mathcal{E}_\CC \neq \emptyset$ for all $6 \times 7$ matrices $U$ of rank six since $\ker(Z') \subseteq \ker(UZ')$. This implies that,
for a $6 \times 7$-matrix $U$ of variables,
 the Chow form of $\mathcal{E}_\CC$ must vanish when evaluated at the $6 \times 9$ matrix $UZ'$.

 Consider $\textup{Ch}_{\mathcal{E}_\CC}(UZ')$ as a polynomial in the
 $42$ variable entries of $U$. Let $\mathcal{P}$ be the set of coefficients of that polynomial.
 Thus $\mathcal{P}$ is a set of polynomials of degree $60$
   in the $63$ entries of $Z'$.  All these polynomials must vanish   when
$\ker_\CC(Z') \cap \mathcal{E}_\CC \neq \emptyset$.

\begin{theorem} The variety $\ker_\CC(Z) \cap \mathcal{E}_\CC $ is non-empty in $\PP_\CC^9$
if and only if, for all $7 \times 9$ submatrices $Z'$ of $Z$ of rank seven,
we have $p(Z') = 0$ for all $p \in \mathcal{P}$.
\end{theorem}

\begin{proof}
When $\rank(Z)=7$, we are intersecting $\mathcal{E}_\CC$, which is of codimension three, with a projective line in $\PP_\CC^8$.
We argued above that if a line meets $\mathcal{E}_\CC$ then every plane containing the line also meets $\mathcal{E}_\CC$. Conversely, if every plane containing a line meets  $\mathcal{E}_\CC$ then so does the line. This converse statement follows from Proposition 3.1 in~\cite{dalbec-sturmfels} after a projection onto $\PP_\CC^7$.
\end{proof}

In the case $\rank(Z)=7$, there is a simpler way to express
  the precise conditions under which
$\ker_\CC(Z) \cap \mathcal{E}_\CC \neq \emptyset$. This involves  the Chow form of the
cubic curve $\mathcal{C}$ from Example~\ref{ex:twisted cubic chow form}, and in fact, leads to an algorithm to decide when $Z$ has a (real) essential matrix. We describe the details now.

Suppose that $ u,v\in \RR^9$ is a basis of  ${\rm ker}_\CC(Z)$. Since any element of
$\ker_\CC(Z)$ is of the form $\lambda u + \mu v$ for $\lambda, \mu \in \CC$, we have that
$\ker_\CC(Z) \cap \mathcal{E}_\CC \neq \emptyset$ if and only if
there exist $\lambda, \mu \in \CC$ such that
$$p_j (\lambda u + \mu v )\,\, =\,\, r_{j1} \lambda^3 + r_{j2} \lambda^2 \mu + r_{j3}  \lambda \mu^2 + r_{j4} \mu^3
\,\, = \,\,0 \qquad \hbox{for} \,\,j = 1,2,\ldots,10 $$
where $p_j$ is a cubic in (\ref{eq:tracecubics})-(\ref{eq:det}).
We may express all 10 equations together as
\begin{align} \label{eq:matrix R}
R \omega = 0 \,\,\,\textup{ where } R = (r_{jl}) \in \RR^{10\times 4}
\end{align}
and $\omega :=
\begin{bmatrix} \lambda^3 \!&\! \lambda^2 \mu \!&\! \lambda \mu^2 \!&\! \mu^3 \end{bmatrix}^\top$
parametrizes the cubic curve $\mathcal{C} \subset \PP_\CC^3 $ in Example~\ref{ex:twisted cubic chow form}.
This leads to the following useful facts.

\begin{lemma} \label{lem:twisted}
\begin{enumerate}
\item
$\ker_\CC(Z) \cap \mathcal{E}_\CC \neq \emptyset$ if and only if
$\dis
{\rm ker}_{\CC}(R) \cap  \mathcal{C} \neq \emptyset
$.
\item
Suppose $\rank(R) = 2$ and let $R'$ be a $2 \times 4$ submatrix of $R$ that has rank $2$.
 Then ${\rm ker}_{\CC}(R) \cap \mathcal{C} \neq \emptyset $ if and only if
${\rm Ch}_{\mathcal{C}} (R') = 0$.
\end{enumerate}
\end{lemma}

We can now use the above lemma to decide if $Z$ has a (complex) essential matrix.

\begin{corollary} \label{cor:rank7complexE}
We have $\ker_\CC(Z) \cap \mathcal{E}_\CC \neq \emptyset$ if and only if
one of the following holds:
\begin{enumerate}
\item $\rank(R) \leq 1$.
\item $\rank(R) =2$ and   ${\rm Ch}_{\mathcal{C}} (R') = 0$ for any $2\times 4$ submatrix $R'$ of $R$.
\item $\rank(R) = 3$ and the unique point $\omega \in \ker_\CC(R)  \subset \PP_\CC^3 $
satisfies
\begin{equation}
\label{eq:TwoByThree}
 {\rm rank} \begin{bmatrix}  \omega_1  & \omega_2 & \omega_3 \\
                                                  \omega_2  & \omega_3 & \omega_4
\end{bmatrix}  \,\,\,= \,\,\, 1.
\end{equation}
\end{enumerate}
\end{corollary}

\begin{proof} By Lemma~\ref{lem:twisted} (1), a necessary condition for $\ker_\CC(Z) \cap \mathcal{E}_\CC \neq \emptyset$  is that $\rank(R) \leq 3$. If $\rank(R) \leq 1$, then $\ker_\CC(R)$ has dimension at least two as a projective variety and hence will intersect $\mathcal{C}$ by Theorem~\ref{thm:intersections}. If $\rank(R)=2$ then we invoke Lemma~\ref{lem:twisted} (2). If $\rank(R)=3$, then $\ker_\CC(R)$ consists of a unique point $\omega \in  \PP_\CC^3$ and $\ker_\CC(Z) \cap \mathcal{E}_\CC \neq \emptyset$  if and only if this $\omega \in \mathcal{C}$.
The $2 \times 2$ minors of the matrix in \eqref{eq:TwoByThree}
are the polynomials that cut out $\mathcal{C}$, and hence, $\omega \in \mathcal{C}$ if and only if the rank condition in \eqref{eq:TwoByThree} holds.
\end{proof}

In fact, the following algorithm will determine if $Z$ has a (real) essential matrix.
\newpage
\begin{algorithm}
\caption{Existence check for an essential matrix when $\rank(Z)=7$}
\begin{algorithmic}
\Require{$Z \in \RR^{m \times 9}, \,\,\rank(Z)=7, \,\, \{u,v\} \subset \RR^9 \textup{ basis of } \ker_\CC(Z) $}
\State{Compute the $10 \times 4$ matrix $R$ in \eqref{eq:matrix R}.}

\If {$\rank(R)=1$}
\State{Return True}
\EndIf

\If {$\rank(R)=4$}
\State{Return False}
\EndIf

\If {$\rank(R)=3$}
  \State{Compute $\omega \in \ker_\RR(R)$}
  \If {$\rank \begin{bmatrix}  \omega_1  & \omega_2 & \omega_3 \\
                                                  \omega_2  & \omega_3 & \omega_4
                   \end{bmatrix}  \,\,\,\,= \,\,\, 1 $}
  \State{Return True}
  \EndIf
    \State{Return False}

\EndIf

\If {$\rank(R)=2$}
     \State{Pick two independent rows $r_i,r_j$ of $R$ and set up the B{\'e}zout matrix $B$ in \eqref{eq:Bezout}}
         \If {$\rank(B) = 2$}
         \State{Return True}
         \EndIf

           \State{
           $at^2 + bt +c = \operatorname{gcd}(
        r_{i1} t^3 + r_{i2}t^2 + r_{i3}t  + r_{i4},\,\,\,
        r_{j1} t^3 + r_{j2}t^2 + r_{j3}t  + r_{j4})$
        }
              \If {$b^2-4ac \geq 0$}
              \State{Return True}
              \EndIf
              \State{Return False}
\EndIf

 \end{algorithmic}
\label{algo:existenceOfErankZ=7}
\end{algorithm}

\noindent{\em Proof of correctness of Algorithm~\ref{algo:existenceOfErankZ=7}}:
Consider the three cases in Corollary~\ref{cor:rank7complexE}.
\begin{enumerate}
\item If $\rank(R)= 1$ then $\ker_\CC(R) \cap \mathcal{C} \neq \emptyset$ if and only if
there exists $\lambda, \mu \in \CC$ such that $r_{i1} \lambda^3 + r_{i2} \lambda^2 \mu + r_{i3} \lambda \mu^2 + r_{i4} \mu^3 = 0$ where $(r_{i1},r_{i2},r_{i3},r_{i4})$ is a non-zero row of $R$. Setting $\mu = 1$, we obtain
a univariate cubic equation, which always has a real root $\lambda$. Since the basis vectors $u,v \in \RR^9$ of $\ker_\CC(Z)$ were real, we get that the real point $\lambda u + \mu v$ satisfies all the Demazure equations. Therefore, $Z$ has an essential matrix.

\item If $\rank(R)=2$, by the same argument as above, we  have two univariate cubics
that share a root (coming from two independent equations in $R \omega = 0$).
If that root is unique then it is real. (This happens
when the $3 \times 3$  B\'ezout matrix in  \eqref{eq:Bezout} has rank exactly $2$.)
Otherwise, the gcd of our two univariate cubics is a quadratic, and we check the sign of its discriminant to
determine if there is a real root.

\item If $\rank(R) = 3$ then the unique $\omega$ in $\ker_\CC(R)$ can be chosen
to have real coordinates, and ditto for its preimage $(\lambda,\mu)$
under the parametrization of the curve $\mathcal{C}$. So again, $Z$ has an essential matrix.
\end{enumerate}
\hfill $\Box$

\section{Discussion}
\label{sec:discussion}
The study of fundamental and essential matrices has a long history in computer vision and photogrammetry~\cite{hartley-zisserman-2003,sturm11}.  However, this literature has primarily been focused on a few topics:
the existence problem in the {\em minimal} case (i.e., $m=5$ for essential matrices and $m=7$ for fundamental matrices),
estimation methods, and degenerate configurations in which these estimation methods break down.

We have presented a complete existence result for fundamental matrices. It describes a set of polynomial equations and inequations that must be satisfied for a fundamental matrix to exist,
and it establishes a polynomial time algorithm (in $m$) for performing this check. It is a constructive (though not necessarily a numerically stable) procedure. An important point in our
analysis is the characterization of the non-trivial gap between the constraints $\rank(F) = 2$ and $\det(F) = 0$.

For essential matrices, we gave a complete answer for $\rank(Z) \leq 3$ and $\rank(Z) \geq 7$. For $4 \le \rank(Z) \le 6$ we
tackle the first step of finding complex solutions to the Demazure polynomials, but more work needs to be done to determine the existence of essential matrices which are real solutions to the Demazure cubics. Also, a more
explicit construction of the Chow form of $\mathcal{E}_\CC$ needs to be worked out to make our checks constructive.

Going forward there are a number of avenues for future work.
The first, of course, is to complete the characterization of essential matrices for all ranks of $Z$. Another direction
is  the analogous existence question for trifocal and quadrifocal tensors.

The results presented in this paper are mostly in terms of the null space of the matrix $Z$. It is straightforward to translate them in terms of $X$ and $Y$, but this does not pay any attention to the fact that each row of $Z$ is a Kronecker product of the corresponding rows of $X$ and $Y$. We believe that exploiting this significant structure will lead to new and simpler forms of our results.
Last but not the least, we have ignored {\em chirality}, i.e., the constraint that the correspondences can be triangulated to 3D points that lie in front of the cameras. This is an interesting problem and will likely require significantly different techniques to deal with its semialgebraic nature.

\bibliographystyle{apalike}
\bibliography{algebraic_vision}
\end{document}